\documentclass[letterpaper,12pt,headings=normal,parskip=half]{scrartcl}

\usepackage[american]{babel}            % American English hyphenation
\usepackage[utf8]{inputenc}             % UTF-8 input encoding
\usepackage[T1]{fontenc}                % T1 font encoding
\usepackage{lmodern}                    % Improved font rendering

\usepackage{geometry}                   % Page geometry
\usepackage{fullpage}                   % Sets smaller margins
\usepackage[x11names]{xcolor}           % Comprehensive color package

\usepackage{amsmath, amssymb, amsthm}   % Math environments and symbols
\usepackage{mathtools}                  % Enhancements for amsmath
\usepackage{dsfont}                     % Double-struck fonts (e.g., \mathds{1})

\usepackage{graphicx}                   % Include images
\usepackage{caption}                    % Customize captions
\usepackage{subcaption}                 % Subfigures
\usepackage{booktabs}                   % Nicer tables
\usepackage{tabulary}                   % Flexible table columns

\usepackage{authblk}                    % Author and affiliation handling
\usepackage{xspace}                     % Smart spacing after macros
\usepackage{enumitem}                   % Customizable lists
\usepackage{tikz}                       % Drawing package
\usetikzlibrary{calc}
\usetikzlibrary{arrows.meta}
\usetikzlibrary{positioning}
\usetikzlibrary{decorations.text}
\usetikzlibrary{decorations.markings}
\addtokomafont{sectioning}{\rmfamily}   % Section fonts in roman
\usepackage{natbib}
\usepackage{hyperref}
\usepackage[capitalize,noabbrev]{cleveref}
\hypersetup{
  colorlinks=true,
  linkcolor=navy,
  citecolor=MPIgreen!80!white,
  urlcolor=black,
  filecolor=black
}

\mathcode`@="8000
{\catcode`\@=\active\gdef@{\mkern1mu}}

\DeclareMathOperator{\supp}{supp}
\DeclareMathOperator{\spn}{span}
\DeclareMathOperator{\aff}{aff}
\DeclareMathOperator{\conv}{conv}

\makeatletter
\def\moverlay{\mathpalette\mov@rlay}
\def\mov@rlay#1#2{\leavevmode\vtop{%
   \baselineskip\z@skip \lineskiplimit-\maxdimen
   \ialign{\hfil$\m@th#1##$\hfil\cr#2\crcr}}}
\newcommand{\charfusion}[3][\mathord]{
    #1{\ifx#1\mathop\vphantom{#2}\fi
        \mathpalette\mov@rlay{#2\cr#3}
      }
    \ifx#1\mathop\expandafter\displaylimits\fi}
\makeatother
\newcommand{\cupdot}{\charfusion[\mathbin]{\cup}{\cdot}}

\newcommand{\hccpwl}{\mathcal{C}^+} % ph convex cpwl
\newcommand{\hcpwl}{\mathcal{C}} % ph nonconvex cpwl
\newcommand{\MAX}{\mathcal{M}} % combinations of maxima
\newcommand{\networksl}[2]{\check{\mathcal{N}}_{#1}({#2})} % sparse maxout
\newcommand{\networks}[2]{\mathcal{N}_{#1}({#2})} % sparse maxout

\newcommand{\Polytopes}{\mathcal{P}} % polytopes
\newcommand{\vPolytopes}{\mathcal{V}} % virtual polytopes
\newcommand{\dualpolytopesl}[2]{\check{\mathcal{V}}_{#1}({#2})} % dual polytopes
\newcommand{\dualpolytopes}[2]{\mathcal{V}_{#1}({#2})} % dual polytopes

\newcommand{\cpwl}{continuous piecewise linear\xspace}
\newcommand{\relu}{ReLU\xspace}
\newcommand{\dvec}{\boldsymbol{d}}      % degree vector
\newcommand{\rvec}{\boldsymbol{r}}      % rank vector
\newcommand{\R}{\mathbb{R}}
\newcommand{\N}{\mathbb{N}}

\theoremstyle{definition}
\newtheorem{theorem}{Theorem}
\newtheorem{definition}[theorem]{Definition}
\newtheorem{lemma}[theorem]{Lemma}

\newtheorem{corollary}[theorem]{Corollary}

\makeatletter
\def\thm@space@setup{\thm@preskip=0pt
\thm@postskip=0pt}
\makeatother

\definecolor{navy}{RGB}{38, 70, 83}
\definecolor{cyan}{RGB}{65, 138, 127}
\definecolor{gold}{RGB}{233, 196, 106}
\definecolor{apri}{RGB}{244, 162, 97}
\definecolor{clay}{RGB}{231, 111, 81}
\definecolor{MPIgreen}{RGB}{0,108,102}

\tikzstyle{vertex} = [draw=black, fill=black, inner sep=0.75mm, circle]
\tikzstyle{edge} = [black, thick]
\tikzstyle{arc} = [black, thick, shorten >=2mm, decoration={markings, mark=at position 1 with {\arrow{Triangle[length=2mm,width=0.8mm]};}}, postaction={decorate}]
\tikzstyle{labeledvertex} = [draw=black, thick, minimum size=4.5mm, inner sep=0pt, text centered, align=center, circle]
\tikzstyle{maxoutunit} = [draw, thick, minimum width=0.4cm, minimum height=0.8cm,align=center, rounded corners=3pt, fill=none, label=center:\text{\begin{tikzpicture}
    \node (1) at (0,0) {};
    \node (2) at (0.5,0) {};
    \draw[thick, line cap=round] (1) -- (2);
\end{tikzpicture}}]

\makeatletter
\patchcmd{\NAT@citex}
  {\@citea\NAT@hyper@{%
     \NAT@nmfmt{\NAT@nm}%
     \hyper@natlinkbreak{\NAT@aysep\NAT@spacechar}{\@citeb\@extra@b@citeb}%
     \NAT@date}}
  {\@citea\NAT@nmfmt{\NAT@nm}%
   \NAT@aysep\NAT@spacechar\NAT@hyper@{\NAT@date}}{}{}

\patchcmd{\NAT@citex}
  {\@citea\NAT@hyper@{%
     \NAT@nmfmt{\NAT@nm}%
     \hyper@natlinkbreak{\NAT@spacechar\NAT@@open\if*#1*\else#1\NAT@spacechar\fi}%
       {\@citeb\@extra@b@citeb}%
     \NAT@date}}
  {\@citea\NAT@nmfmt{\NAT@nm}%
   \NAT@spacechar\NAT@@open\if*#1*\else#1\NAT@spacechar\fi\NAT@hyper@{\NAT@date}}
  {}{}
\makeatother

\makeatletter
\newlength{\negph@wd}
\DeclareRobustCommand{\negphantom}[1]{%
  \ifmmode
    \mathpalette\negph@math{#1}%
  \else
    \negph@do{#1}%
  \fi
}
\newcommand{\negph@math}[2]{\negph@do{$\m@th#1#2$}}
\newcommand{\negph@do}[1]{%
  \settowidth{\negph@wd}{#1}%
  \hspace*{-\negph@wd}%
}
\makeatother

\title{\Large On the expressivity of sparse maxout networks}
\author{%
\begin{minipage}[t]{0.48\textwidth}\centering
\normalsize \textbf{Moritz Grillo}\\
\footnotesize Max Planck Institute for\\ Mathematics in the Sciences\\
\footnotesize \texttt{moritz.grillo@mis.mpg.de}
\end{minipage}\hfill
\begin{minipage}[t]{0.48\textwidth}\centering
\normalsize \textbf{Tobias Hofmann}\\
\footnotesize TU Berlin\\
\footnotesize \texttt{tobias.hfm@icloud.com}
\end{minipage}
}
\date{}

\begin{document}
\maketitle

\begin{abstract}
\noindent
\textbf{Abstract.} We study the expressivity of sparse maxout networks, where each neuron takes a fixed number of inputs from the previous layer and employs a, possibly multi-argument, maxout activation. This setting captures key characteristics of convolutional or graph neural networks.
We establish a duality between functions computable by such networks and a class of virtual polytopes, linking their geometry to questions of network expressivity.
In particular, we derive a tight bound on the dimension of the associated polytopes, which serves as the central tool for our analysis.
Building on this, we construct a sequence of depth hierarchies. While sufficiently deep sparse maxout networks are universal, we prove that if the required depth is not reached, width alone cannot compensate for the sparsity of a fixed indegree constraint. \\



\end{abstract}

\section{Introduction}\label{sec:intro}

Despite the impressive progress that has been made in the practical use of neural networks, our theoretical understanding of their inner mechanisms and why they generalize so well is still far from being complete.
Universal approximation theorems, such as those of \citet{cybenko1989approximation}, \cite{hornik1990universal}, or \cite{leshno1993multilayer} marked major steps towards explaining the expressivity of neural networks.
Yet, making precise what neural networks can or cannot do goes beyond approximation results and requires addressing questions of exact representability.
That is, given a specific neural network architecture, what is the exact class of functions that can be represented by that architecture?
Recent work, for example, \cite{telgarsky2016benefits}, \cite{arora2018understanding}, or \cite{hertrich2021towards}, have provided new insights into how architectural parameters such as depth and width influence expressivity.

In contrast, this article investigates how network sparsity affects expressivity and how it interacts with architectural parameters such as depth and width.
Beyond that, we are interested in the effects of using multi-argument maxout activation functions, which can be seen as generalized rectified linear units, as introduced by \cite{goodfellow2013maxout}.
We call a network sparse, or more specifically indegree-$d$-constrained, if each of its neurons depends on a fixed number $d$ of outputs from the previous layer. This framework interpolates between fully connected architectures ($d$ equal to the previous layer’s width) and extremely sparse ones ($d=2$).
A precise formalization of the architecture we are concerned with follows in Section~\ref{sec:sparse_networks}.
Our considerations are motivated by two aspects.

First, one might find it interesting that explicit attempts to construct neural networks of minimum depth, such as those by \cite{arora2018understanding} or the recent improvement by \cite{bakaev2025better} are nonetheless sparse.
In other words, when not restricting the width of the neural network, sparse networks can compute any \cpwl function with logarithmic depth. 

Second, sparsity is a characteristic property of many practically relevant network types.
For example, a convolutional layer involving a typical $3\times 3$ kernel is indegree-$9$-constrained.
Another example are graph neural networks, which, depending on the specific variant, often come with indegree constraints stemming from their localized feature update rules.
The preceding examples also provide relevant network types in which indegree-constrained operations are used together with pooling steps. In our model, this is reflected by incorporating multi-argument maxout activations.

Note also that the architecture we investigate allows for varying indegree constraints across layers.
The only exception is the first and last layer, where indegree constraints are not imposed.
Thus, our model interpolates between fully connected networks on those in which each neuron depends on only two inputs from the previous layer, apart from the first and last one.
Not enforcing indegree constraints in the first and last layer has the following reasons.

From a practical perspective, even networks that otherwise incorporate sparse operations, such as convolutional or graph neural networks often bundle information via fully connected layers into one or few neurons in the final layer.
See \cite{krizhevsky2012imagenet} for a prominent example.
Allowing the first layer to be fully connected can be seen as allowing a general feature transformation or preprocessing step, for instance, in which an image is patched or graph nodes are embedded into a suitable feature space.

Another reason is that our theoretical interest lies more on structural effects of sparsity across the layers of a network, rather than on the information-theoretic limitations imposed by a restricted receptive field in the first layer.
For example, given an input with $n$ pixels and a network in which each neuron, including those in the first layer, depend on at most two inputs from the preceding layer, it is clear that a network with depth smaller than $\log_2(n)$ is insufficient to capture every aspect of the input.
But this says little about structural effects within the network, but merely reflects the restrictive reception in the beginning of information processing chain. 

\paragraph{Our Contributions.}
Using a duality between positively homogeneous convex \cpwl functions and polytopes, we establish an isomorphism between functions computed by sparse maxout networks and a class of \emph{virtual polytopes} (\Cref{thm:iso_sparse_virtual}), which are formal differences of polytopes.
We then study the dimensions of these virtual polytopes in \Cref{sec:dimension}, where \Cref{thm:dim_bound} provides an upper bound and \Cref{thm:dimboundattained} shows that this bound is attained.
These results form the foundation of our analysis of how sparsity affects expressivity in \Cref{sec:expressivity}.

While our general framework, introduced in Section~\ref{sec:sparse_networks}, allows for varying indegree and rank constraints across layers, and thus involves vector-valued parameters, the following outline of our main results suppresses this notation for simplicity. For any indegree constraint~$d\in \N$, number of arguments~$r\in\mathbb{N}$ in the maxout activation, and input dimension~$n$, let $\networks{n}{d,r,\ell}$ denote the class of functions representable by such networks with~$\ell$ hidden layers. Our main results are as follows:
\begin{itemize}
\item For any~$d, r$, and~$\ell$, there exist an input dimension~$n$ and a function~$f \colon \R^n \to \R$ computable by a fully connected network with two hidden layers and width~$O(n)$ that is not contained in~$\networks{n}{d,r,\ell}$ (\Cref{thm:widthcannotcompensate}). 
This, in particular, implies that sparse shallow networks cannot represent all \cpwl functions, regardless of their width.
\item For~$d = r = 2$, we fully characterize the class of functions~$\networks{n}{2,2,\ell}$ (\Cref{thm:Nl22max2^l}), showing a sharp separation for each~$\ell \in \{1, \ldots, \lceil \log_2 (n+1) \rceil\}$.
\end{itemize}

In words, when allowing arbitrary input dimension, for any fixed depth and indegree constraint, the resulting sparsity cannot be compensated by increasing width.
This, in particular, also implies that for a fixed indegree, there is no depth that fully compensates for sparsity.
Overall, our results demonstrate that sparsity has a decisive impact on the expressivity of neural networks.

\paragraph{Further Related Work}
An extensive body of research has examined the effect of depth \cite{hertrich2021towards,haase2023lower,averkov2025on,grillo2025depthboundsneuralnetworksbraid,bakaev2025better} as well as the relationship between depth and size in neural networks, showing that deeper architectures can represent certain functions exponentially more compactly than shallow ones \cite{montufar14number,telgarsky2016benefits,arora2018understanding,ergen24topology}.  Another perspective considers bounds on the total size of the networks \cite{hertrich2023provably,hertrich2024relu,10.5555/3600270.3600790,brandenburg2024decompositionpolyhedrapiecewiselinear}.

A related methodological development connects questions of expressivity to Newton polytopes of functions realized by neural networks through tropical geometry. This approach was initiated by \cite{zhang2018tropical} and further developed by \cite{maragos2021tropical}. It has since been applied to the study of decision boundaries, depth and size lower bounds, the number of linear regions, and approximation capabilities \cite{montufar2022sharp,misiakos2022neural,haase2023lower,brandenburg2024real,valerdi2024minimal,koutschan2023representing,hertrich2025neuralnetworksvirtualextended,balakin2025maxoutpolytopes}.

\section{Sparse networks and \cpwl functions}\label{sec:sparse_networks}

A \emph{polyhedron}~$P$ is the intersection of finitely many closed halfspaces and a \emph{polytope} is a bounded polyhedron.
A hyperplane \emph{supports}~$P$ if it bounds a closed halfspace containing~$P$, and
any intersection of~$P$ with such a supporting hyperplane yields a \emph{face}~$F$ of~$P$. 
A \emph{polyhedral complex} $\mathcal{P}$ is a finite collection of polyhedra such that (i)~$\emptyset \in \mathcal{P}$, (ii)~if~$P \in \mathcal{P}$, then all faces of $P$ are in $\mathcal{P}$, and (iii)~if $P,Q \in \mathcal{P}$, then $P \cap Q$ is a face both of $P$ and $Q$.
For additional background on polyhedral geometry, we refer to \cite{schrijver1998theory} and \cite{ziegler2012lectures}.
A continuous function $f\colon\R^n\to\R$ is called \emph{\cpwl} if there exists a polyhedral complex $\mathcal{P}$ such that the restriction of $f$ to each full-dimensional polyhedron $P\in\mathcal{P}$ is an affine function.
For a convex \cpwl function $f$, there is a unique coarsest such polyhedral complex $\mathcal{P}$, that is, $f$ is affine within each maximal $P\in\mathcal{P}$, but not within $Q\supsetneq P$ for any $P\in\mathcal{P}^n$. We call $\mathcal{P}$ \emph{the} polyhedral complex \emph{underlying} $f$ 

A function $f\colon\R^n\to\R$ is called \emph{positively homogeneous} if $f(\lambda x)=\lambda f(x)$ for any $\lambda\geq 0$.
We denote by $\hcpwl_n$ the vector space of all positively homogeneous \cpwl functions from $\R^n$ to $\R$ and by $\hccpwl_n$ the set of convex positively homogeneous \cpwl functions.

Maxout activation functions can be regarded as multi-argument generalizations of \relu activation functions.
For $r,m\in\N_0\coloneqq\N\cup\{0\}$, a \emph{rank-$r$-maxout layer} having $n$ inputs and $m$ outputs with weights $a_{ij} \in \R^n$, $i\in[m]$, $j\in[r]$ computes a function $\R^n\to \R^m$,
\begin{equation*}
    x\mapsto \big[\max\limits_{j\in\,[r]}\{a_{ij}^\top\,x\}\big]_{i=1}^m.
\end{equation*}
Denoting the \emph{support} of a vector $x\in\R^n$ by $\supp(x)\coloneqq\{i\in[n]:x_i\neq 0\}$, we call a layer \emph{indegree-$d$-constrained}, for some $d \in \N$, if $|\bigcup_{j \in [r]} \supp(a_{ij})| \leq d$ for all $i\in[m]$.
For $\ell \in\N_0$, $\rvec \in \N^\ell$, a rank-$\rvec$-maxout network with $\ell$ hidden layers computes a function $\R^{n_0} \to \R$, $f_{\ell+1} \circ \ldots \circ f_1$ where $f_i$, $i\in[\ell]$, is a rank-$r_i$-maxout layer and $f_{\ell+1}$ is a linear map from $\R^{n_\ell}$ to $\R$.
We denote the class of functions $f\colon\R^n\to \R$ that are representable by such an architecture by~$\networks{n}{\ell,\rvec}$.
For $\dvec \in \N^\ell$, the network is called \emph{indegree-$\dvec$-constrained} if $f_i$ is indegree-$d_i$-constrained.
The class of functions $f\colon\R^n\to \R$ that are representable by such an architecture is denoted by~$\networks{n}{\ell, \dvec,\rvec}$.
As discussed in the introduction, we allow the first layer to be fully connected, meaning that $d_1=n$, throughout this work.
In our investigation, we often specify $\dvec=(d_1,\ldots,d_\ell),\rvec=(r_1,\ldots,r_\ell)\in\N^\ell$ for a network with $\ell$ hidden layers.
In various arguments, however, we focus on subnetworks with~$\ell'<\ell$ hidden layers.
Then, strictly speaking, in $\networks{n}{\ell',\dvec,\rvec}$, the dimensions of $\dvec$ and $\rvec$ do not match the number of hidden layers $\ell'$.
To avoid overloading the notation, we nevertheless adopt the notation $\networks{n}{\ell',\dvec,\rvec}=\networks{n}{\ell',(d_1,\ldots,d_{\ell'}),(r_1,\ldots,r_{\ell'})}$.
The same convention is used for $\networks{n}{\ell,\rvec}$ and further polytope classes depending on $\dvec$ and $\rvec$, which are introduced in Section~\ref{sec:polytopesandsupportfunctions}.
To denote indegree and rank constraints succinctly, we use boldface numbers for vectors, of appropriate size, containing only the specified number, except that we retain $d_1=n$ if we denote indegree constraints.
For example,~$\dvec=\mathbf{2}$ is to be read as~$\dvec=[n,2,\ldots,2]$ and $\rvec=\mathbf{2}$ is to be read as $\rvec=[2,\ldots,2]$.
For $\ell=0$, there is no hidden layer, and $\networks{n}{0, \dvec,\rvec}=\networks{n}{0, \rvec}$ consists of the linear functions from $\R^n$ to $\R$.
An example for a network from~$\networks{n}{\ell, \textbf{2},\textbf{2}}=\networks{n}{\ell,[n,2,\ldots,2],\textbf{2}}$ is given in Figure~\ref{fig:network}.
One goal of our work is to understand which functions lie in $\networks{n}{\ell, \dvec,\rvec}$. We know that members of this class are positively homogeneous \cpwl, as they are compositions of such functions. To see whether a converse inclusion holds, for certain parameters $\ell$, $\dvec$, and $\rvec$, we rely on the following result by \cite{wang2005generalization}.

\begin{figure}
\centering
\includegraphics[]{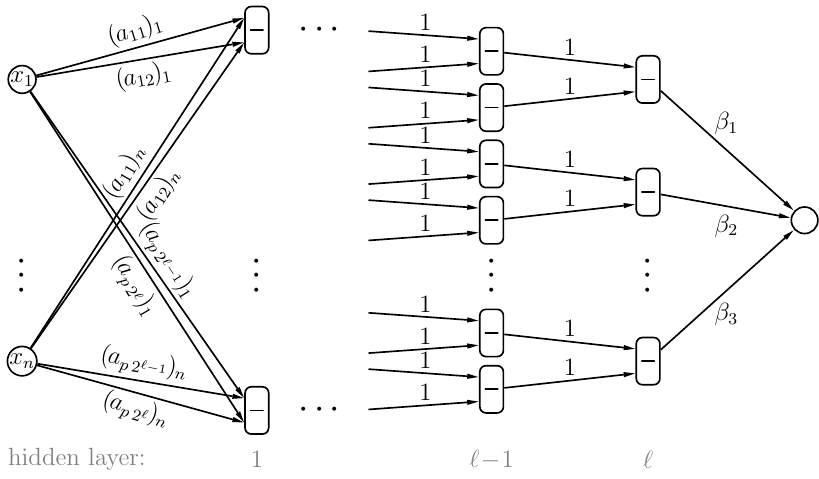}
    \caption{A network from the class $\networks{n}{\ell,\textbf{2},\textbf{2}}$}
    \label{fig:network}
\end{figure}

\begin{theorem}[\cite{wang2005generalization}]\label{thm:WangSun} Every function $f\in\hcpwl_n$ can be written as
\begin{equation*}
    f(x) = \sum_{i=1}^p \beta_i\max\{a_{i,1}^\top x,\ldots,a_{i,n+1}^\top x\}
\end{equation*}
for suitable $p\in\N$, $\beta_i\in\R$, and $a_{i,j}\in\R^n$, $i\in[p]$, $j\in[n+1]$.
\end{theorem}

Building on this representation, \cite{arora2018understanding} give a one-to-one correspondence between the class of functions computable by \relu neural networks and \cpwl functions.
The following result is a refinement of their result, implying that deep indegree-constrained networks are as expressive as networks that are not indegree-constrained.

\begin{theorem}[C.f. \cite{arora2018understanding}]\label{thm:deep_sparse_universal}
For given $\dvec,\rvec\in\N^\ell$, if $n < \prod_{i=1}^\ell \min\{d_i,r_i\}$, then
\begin{equation*}
    \networks{n}{\ell, \rvec} = \networks{n}{\ell,\dvec, \rvec}=\hcpwl_n.
\end{equation*}
\begin{proof}
By definition, every function in $\networks{n}{\ell, \rvec}$ or $\networks{n}{\ell,\dvec, \rvec}$ is in $\hcpwl_n$.
Since, of course, $\networks{n}{\ell,\dvec, \rvec}\subseteq \networks{n}{\ell,\rvec}$, we need to verify that $\hcpwl_n\subseteq\networks{n}{\ell,\dvec, \rvec}$.
By Theorem~\ref{thm:WangSun}, every function $f\in\hcpwl_n$ is a linear combination of maxima of at most $n+1$ linear functions, $f(x) = \sum_{i=1}^p \beta_i\max\{a_{i,1}^\top x,\ldots,a_{i,n+1}^\top x\}$.
In the following, we argue that such a function can be represented by a function $f_{\ell+1}\circ\ldots\circ f_1\in \networks{n}{\ell,\dvec, \rvec}$.
The inner linear functions $a_{i,j}^\top\,x$, $i\in[p]$, $j\in[n+1]$, can be implemented up to the first hidden layer, because $f_1$ involves a general linear transformation without indegree constraints.
The outer linear combination can be represented by the linear function $f_{\ell+1}$.
Finally, the $p$ occurring maxima are implementable by $f_\ell\circ\ldots\circ f_1$, since the maximum function is associative, that is $\max\{x,y,z\}=\max\{\max\{x,y\},z\}$ for all $x,y,z\in \R$.
Hence the maximum of $n+1$ inputs can be obtained sequentially by composing maxima across layers, with all weights set to one.
For $\dvec=\rvec=\mathbf{2}$, this is illustrated in Figure~\ref{fig:network}.
With each layer $i\in[\ell]$, the number of inputs that can be processed grows by a factor of $\min\{d_i,r_i\}$, since each neuron takes at most $d_i$ inputs and computes a maximum over at most $r_i$ inputs.
In total, $n+1$ inputs can be compared if $n+1\leq \prod_{i=1}^\ell\min\{d_i,r_i\}$, as is our claim.
\end{proof}
\end{theorem}

Theorem~\ref{thm:deep_sparse_universal} says that, for any given input dimension, indegree and rank constraint, sparse maxout networks are universal if they are sufficiently deep.
Nevertheless, in Theorem~\ref{thm:widthcannotcompensate} it is shown that for any given depth $\ell\in\N_0$, indegree and rank constraints $\dvec,\rvec\in\N^\ell$, there is an input dimension $n\in\N$ and function $f\in\networks{n}{2,\mathbf{2}}$ which is not in $\networks{n}{\ell,\dvec,\rvec}$.
In words, sparsity has a decisive restrictive effect on the expressivity of a network.
Our work also provides particular insight into the class $\networks{n}{\ell,\mathbf{2},\mathbf{2}}$. We characterize it by comparing it with the vector space
\begin{equation*}
    \MAX_n(m)\coloneqq\big\{f\colon\R^n\to\R\,:\,f(x)=\sum_{i=1}^p \beta_i\max\limits_{j\in[m]}\{a_{ij}^\top x\},~p\in\N,~\beta_i\in\R,~a_{ij}\in\R^n\big\}.
\end{equation*}
For any $n\in\N$ and $\ell\in\N_0$, Theorem~\ref{thm:Nl22max2^l} says that
\begin{equation*}
    \networks{n}{\ell,\mathbf{2},\mathbf{2}}=\MAX_n(2^\ell).
\end{equation*}
The arguments leading to those results rely largely on polyhedral geometry.

\section{Virtual polytopes and support functions}\label{sec:polytopesandsupportfunctions}

The goal of this section is to establish a duality between specific virtual polytopes and the functions computable by related sparse maxout networks. For on overview of virtual polytopes, we refer to \cite{panina2015virtual}.

The Minkowski sum of two sets $X,Y\in\R^n$ is $X+Y=\{x + y: x\in X, y\in Y\}$.
By $\spn(X)$ we denote the smallest linear subspace containing~$X$, by $\aff(X)$ we denote the smallest affine subspace containing $X$, and by $\conv(X)$ we denote the convex hull of $X$.
A \emph{polytope} $P$ is the convex hull of finitely many points in $\R^n$.
We denote by $\Polytopes_n$ the set of all such polytopes.
The \emph{dimension} of a polytope $P$ is the dimension of its affine hull $\dim(\aff(P))$.
Note that $P+Q+(-Q)\neq P$ for $-Q\coloneqq\{-q:q\in Q\}$ and hence naive subtraction is not the inverse of Minkowski addition.
But there is a canonical way to introduce a formal Minkowski difference.
The set of polytopes together with Minkowski addition forms a commutative semigroup $S$, having the polytope that consists only of the zero vector as neutral element.
For any polytopes $P,Q,R$ in $\Polytopes_n$, there holds the cancellation law, which says that if $P+Q=R+Q$, then $P=R$.
This allows to construct the Grothendieck group of~$S$ by $\vPolytopes_n \coloneqq\{(P, Q) : P, Q \in \Polytopes_n\} / \sim$ where the equivalence relation $\sim$ is defined by
\begin{equation*}
    (P_1, Q_1) \sim (P_2, Q_2) \iff  P_1 + Q_2 = P_2 + Q_1.
\end{equation*}
The equivalence class represented by $(P,Q)$ is denoted by $P-Q$ and called \emph{virtual polytope}.
Furthermore, we denote $\vPolytopes\coloneqq\bigcup_{n\in\N}\vPolytopes_n$.
As a virtual polytope is not a point set, further definitions have to be adjusted appropriately.
For virtual polytopes $P_1-Q_1$ and $P_2-Q_2$ and $\lambda\in\R$, one defines the operations
\begin{align*}
    \lambda(P_1-Q_1)      &\coloneqq (\lambda P_1 - \lambda Q_1),\\
    (P_1-Q_1) + (P_2-Q_2) &\coloneqq (P_1+P_2) - (Q_1+Q_2),\\
    \conv(P_1-Q_1,P_2-Q_2)&\coloneqq \conv(P_1+Q_2,P_2+Q_1)-(Q_1+Q_2).
\end{align*}
For more details about virtual polytopes, we refer to the survey of \cite{panina2015virtual}.
The \emph{support function} $f_P \colon \R^n \to \R$ of a polytope $P = \conv\{a_j : j\in[p]\}$ is given by
\begin{equation*}
    f_P(x)\coloneqq \max\{a_j^\top x : j\in[p]\}.
\end{equation*}
Note that support functions depend only on the extreme points of a given polytope.
So they are independent of a specific representation.
More precisely, if $\conv\{a_j : j\in[p]\}$ and $\conv\{b_k : k\in[q]\}$ describe the same polytope, then the corresponding support functions are identical.
In tropical geometry, functions of the form $\max\{a_j^\top x : j\in[p]\}$ are also called \emph{tropical polynomial}.
As for classical polynomials, one defines the \emph{Newton polytope} $P(f)$ of $f(x) = \max\{a_j^\top x : j\in[p]\}\in\hccpwl_n$ as
\begin{equation*}
    P(f)\coloneqq \conv\{a_j : j\in[p]\}.
\end{equation*}
The following theorem shows a duality between the functions in $\hccpwl_n$ and their Newton polytopes. 
\begin{theorem}[\cite{schneider2013convex}]\label{thm:duality}
The map $\Polytopes_n \to \hccpwl_n, P \mapsto f_P$ is a bijection and for all $P,Q \in \Polytopes_n$ and $\lambda \geq 0$, it holds that 
\begin{enumerate}[topsep=0pt,itemsep=3pt,partopsep=0pt, parsep=0pt]
    \item $f_{P+Q} = f_P + f_Q$,
    \item $f_{\lambda P} = \lambda f_P$, and 
    \item $f_{\conv(P\cup Q)} = \max\{f_P,f_Q\}$.
\end{enumerate}

\end{theorem}
For our investigation it is essential to extend the preceding relationships to virtual polytopes and corresponding functions in $\hcpwl_n$.
Given the support functions $f_P$ and $f_Q$ of two polytopes $P$ and $Q$, the \emph{support function of the virtual polytope} $P-Q$ is defined as $f_{P-Q}\coloneqq f_P-f_Q$.
We also build on the following well-known fact.
\begin{lemma}[\cite{melzer2009expressibility}]\label{lem:cpwl_difference_convex}
Every positively homogeneous \cpwl function $f\colon\R^n\to\R$ can be written as $f=g-h$ where $g$ and $h$ are convex positively homogeneous \cpwl functions.
\end{lemma}
The following Lemma is a direct application of \Cref{thm:duality} and \Cref{lem:cpwl_difference_convex} and the fact that virtual polytopes are constructed as the Grothdendieck group of polyoptes \citep{panina2015virtual}.
\begin{lemma}\label{lem:iso_cpwl_virtual}
The map $\vPolytopes_n\to \hcpwl_n, P-Q\mapsto f_{P-Q}$ is a bijection and satisfies
\begin{enumerate}
    \item $f_{(P-Q)+(R-S)} = f_{P-Q} + f_{R-S}$, \label{lem:iso_cpwl_virtual:i}
    \item $f_{\lambda(P-Q)} = \lambda f_{P-Q}$, \label{lem:iso_cpwl_virtual:ii}
    \item $f_{\conv(P \cup Q)} = \max\{f_P, f_Q\}$. \label{lem:iso_cpwl_virtual:iii}
\end{enumerate}
\begin{proof}
To check that the given map is well-defined, note first that support functions are positively homogeneous and \cpwl.
So $f_{P-Q}= f_P-f_Q$ is also positively homogeneous and \cpwl, for any polytopes $P$ and $Q$.
We yet have to show that $f_{P-Q}$ is independent of the representation of $P-Q$.
Given equivalent virtual polytopes $P-Q$ and $R-S$, we obtain $P+S=R+Q$ and hence $f_{P-Q}(x)=f_P(x)+f_S(x) = f_R(x)+f_Q(x)=f_{R-S}(x)$.

To see that the map $P-Q \mapsto f_P(x)-f_Q(x) $ is injective, let $f_{P-Q}(x)=f_{R-S}(x)$ for all~$x \in \R^n$.
Rearranging gives $f_P(x)+f_S(x) = f_R(x)+f_Q(x)$.
Since the support function uniquely determines a polytope, it follows that $P+S = R+Q$.
Thus, the virtual polytopes $P-Q$ and $R-S$ are the same, as is to be shown.

To verify that the map $P-Q \mapsto f_P(x)-f_Q(x)$ is surjective, consider an arbitrary positively homogeneous \cpwl function $f(x)$ on $\R^n$.
By Lemma~\ref{lem:cpwl_difference_convex}, we can assume that $f$ is of the form $\max\{a_j^\top x: j\in[p]\} - \max\{b_k^\top x : k \in [q]\}$, for finitely many points $a_j$, $j\in [p]$, and $b_k$, $k \in [q]$.
Defining the polytopes $P = \conv\{a_j : j\in [p]\}$ and $Q = \conv\{b_k : k\in [q]\}$, we obtain $f(x) = f_P(x)-f_Q(x)$ and hence every positively homogeneous \cpwl function is the image of the map, establishing surjectivity.

To verify Relation \ref{lem:iso_cpwl_virtual:i}, consider two arbitrary virtual polytopes $P-Q$ and $R-S$ and check that
\begin{align*}
    f_{(P-Q)+(R-S)} &= f_{(P+R)-(Q+S)} = f_{P+R} - f_{Q+S} = (f_P + f_R) - (f_Q + f_S) \\ &= (f_P - f_Q) + (f_R - f_S) = f_{P - Q} + f_{R - S}.
\end{align*}%
For Relation \ref{lem:iso_cpwl_virtual:ii}, take an any virtual polytope $P-Q$ and $\lambda\in \R$.
There holds
\begin{align*}
    f_{\lambda (P - Q)} &= f_{\lambda P - \lambda Q} = f_{\lambda P} - f_{\lambda Q} = \lambda f_{P} - \lambda f_{Q} = \lambda(f_{P}-f_{Q})= \lambda f_{P-Q}.
\end{align*}%
For \ref{lem:iso_cpwl_virtual:iii}, consider arbitrary virtual polytopes $P-Q$ and $R-S$. We find that
\begin{align*}
    f_{\conv(P-Q,@R-S)}
    &= f_{\conv((P+S)\cup (R+Q))-(Q+S)} \\
    &= f_{\conv((P+S)\cup (R+Q))} - f_{Q+S} \\
    &= \max\{f_{P+S},f_{R+Q}\} - f_{Q+S} \\
    &= \max\{f_P+f_S,f_R+f_Q\} - (f_Q + f_S) \\
    &= \max\{f_P+f_S - (f_Q + f_S),f_R+f_Q - (f_Q + f_S)\} \\
    &= \max\{f_P - f_Q, f_R - f_S\} \\
    &= \max\{f_{P-Q}, f_{R-S}\}.\qedhere
\end{align*}
\end{proof}
\end{lemma}

Having established a general duality between virtual polytopes and corresponding functions in $\hcpwl_n$, our next goal is to focus on those virtual polytopes which correspond to the functions computable by sparse maxout networks.

\begin{definition}\label{def:dualpolytopes}
Let be given the parameters $\ell \in\N_0$, $\dvec=(d_1=n,d_2,\ldots,d_\ell)\in\N^\ell$, and $\rvec=(r_1,\ldots,r_\ell)\in\N^\ell$.
We define recursively
\begin{align*}
    \dualpolytopesl{n}{0,\dvec,\rvec} &\coloneqq \{\{v\}:v\in\R^n\}, \\
    \dualpolytopesl{n}{\ell,\dvec,\rvec} &\coloneqq \big\{\conv\left\{\textstyle\sum_{j=1}^{d_{\ell}}\alpha_{ij}V_j : i\in[r_\ell]\right\}\;:~\alpha_{ij}\in\R,~V_j\in \dualpolytopesl{n}{\ell-1,\dvec,\rvec} \big\}.\\
\intertext{Furthermore, let us denote}
    \dualpolytopes{n}{\ell,\dvec,\rvec} &\coloneqq \big\{ \textstyle\sum_{i=1}^k \alpha_i V_i~:~k\in\N,~\alpha_i\in\R,~V_i \in \dualpolytopesl{n}{\ell,\dvec,\rvec}\big\},\\
    \dualpolytopesl{}{\ell,\dvec,\rvec} &\coloneqq  \textstyle\bigcup_{n\in\N} \dualpolytopesl{n}{\ell,\dvec,\rvec}\text{, and}\\
    \dualpolytopes{}{\ell,\dvec,\rvec} &\coloneqq  \textstyle\bigcup_{n\in\N} \dualpolytopes{n}{\ell,\dvec,\rvec}.
\end{align*}
\end{definition}

In \cite{hertrich2021towards}, it is shown that the above recursive definition of virtual polytopes precisely corresponds to the Newton polytopes of the function computed by networks with $\ell$ layers.  The following theorem, carrying this to the indegree-constrained setting, is an immediate consequence
\begin{theorem}\label{thm:iso_sparse_virtual}
The map $V\mapsto f_{V}$ is a bijection between $\dualpolytopesl{n}{\ell,\dvec,\rvec}$ and $\networksl{n}{\ell,\dvec,\rvec}$ and between $\dualpolytopes{n}{\ell,\dvec,\rvec}$ and $\networks{n}{\ell,\dvec,\rvec}$.
\begin{proof}
Note that $\dualpolytopesl{n}{\ell,\dvec,\rvec}\subseteq\dualpolytopes{n}{\ell,\dvec,\rvec}\subseteq\vPolytopes_n$.
Likewise, $\networksl{n}{\ell,\dvec,\rvec}\subseteq\networks{n}{\ell,\dvec,\rvec}\subseteq\hcpwl_n$.
Since Lemma~\ref{lem:iso_cpwl_virtual} says that $V\mapsto f_{V}$ is a bijection between $\vPolytopes_n$ and $\hcpwl_n$, we certainly know that its restrictions to $\dualpolytopesl{n}{\ell,\dvec,\rvec}$ or $\dualpolytopes{n}{\ell,\dvec,\rvec}$ are injective.
It remains to be shown that the image of $\dualpolytopesl{n}{\ell,\dvec,\rvec}$ is $\networksl{n}{\ell,\dvec,\rvec}$.
From there, Lemma~\ref{lem:iso_cpwl_virtual}, Part~\ref{lem:iso_cpwl_virtual:i}, gives directly that $\networks{n}{\ell,\dvec,\rvec}$ is the image of $\dualpolytopes{n}{\ell,\dvec,\rvec}$. The point is that the functions in $\networks{n}{\ell,\dvec,\rvec}$ are exactly the linear combinations of those in $\networksl{n}{\ell,\dvec,\rvec}$, just as the virtual polytopes in $\dualpolytopes{n}{\ell,\dvec,\rvec}$ are precisely the linear combinations of those in $\dualpolytopesl{n}{\ell,\dvec,\rvec}$.

Let us begin with the case where $\ell=0$.
By definition, $\dualpolytopesl{n}{0,\dvec,\rvec}=\{\{v\}:v\in\R^n\}$. The corresponding support functions are the linear functions from $\R^n$ to $\R$, which are precisely the functions in $\networksl{n}{0,\dvec,\rvec}$.

We proceed inductively, assuming that the image of $\dualpolytopesl{n}{\ell-1,\dvec,\rvec})$ is $\networksl{n}{\ell-1,\dvec,\rvec}$.
Any $V\in \dualpolytopesl{n}{\ell,\dvec,\rvec}$, by Definition~\ref{def:dualpolytopes}, is of the form
\begin{equation*}
    V=\conv\left\{\textstyle\sum_{j=1}^{d_{\ell}}\alpha_{ij}V_j : i\in[r_\ell]\right\}
\end{equation*}
for certain $\alpha_{ij}\in\R$ and $V_j\in \dualpolytopesl{n}{\ell-1,\dvec,\rvec}$.
By Lemma~\ref{lem:iso_cpwl_virtual}, Parts~\ref{lem:iso_cpwl_virtual:i} to \ref{lem:iso_cpwl_virtual:iii},
\begin{equation*}
    f_V(x)=\max\limits_{i\in[r_\ell]}\sum\limits_{j=1}^{d_\ell} \alpha_{ij}f_{V_j}(x).
\end{equation*}
This is precisely what is computed at a neuron of a rank-$r_\ell$-maxout layer that is indegree-$d_\ell$-constrained.
Since, by our induction hypothesis, each $f_{V_j}$ is realized by a network from $\networksl{n}{\ell-1,\dvec,\rvec}$, it follows that $f_V$ is in $\networksl{n}{\ell,\dvec,\rvec}$.
Conversely, every function in $\networksl{n}{\ell,\dvec,\rvec}$ is of the same form as $f_V$.
With the $f_{V_j}$ lying in $\networksl{n}{\ell-1,\dvec,\rvec}$, we know by induction that the corresponding $V_j$ belong to $\dualpolytopesl{n}{\ell-1,\dvec,\rvec}$.
These serve to construct $V=\conv\left\{\textstyle\sum_{j=1}^{d_{\ell}}\alpha_{ij}V_j : i\in[r_\ell]\right\}$, which shows that every function in $\networksl{n}{\ell,\dvec,\rvec}$ is indeed the image of a suitable virtual polytope in $\dualpolytopesl{n}{\ell,\dvec,\rvec}$.
\end{proof}
\end{theorem}

Note that the relations \ref{lem:iso_cpwl_virtual:i} to \ref{lem:iso_cpwl_virtual:iii} of Lemma~\ref{lem:iso_cpwl_virtual} remain valid when restricting the map $V\mapsto f_V$.
In particular, this map is not only a bijection between $\networksl{n}{\ell,\dvec,\rvec}$ and $\dualpolytopesl{n}{\ell,\dvec,\rvec}$, as stated in Theorem~\ref{thm:iso_sparse_virtual}, but also a vector space isomorphism and a semigroup isomorphism, the latter with respect to building the convex hull and taking the maximum.

\section{The dimension of polytopes arising from sparse networks}\label{sec:dimension}

Having discussed which virtual polytopes are dual to functions computable by networks in $\networksl{n}{\ell,\dvec,\rvec}$, we now aim to bound the dimension of these polytopes.
Eventually, this will link the geometry of the polytopes to the expressivity of corresponding networks.

\begin{definition}
    The \emph{dimension} of a virtual polytope $V\in\vPolytopes$ is defined by \[\dim(V) = \min \{\dim(P+Q) \mid V=P-Q\}.\]
\end{definition}

The proof of this section's main result, Theorem~\ref{thm:dim_bound}, builds on the following two lemmas.

\begin{lemma}\label{lem:affhull_subspaces}
For subspaces $U_1,\ldots,U_k\subseteq\R^n$ and vectors $x_1,\ldots,x_k\in\R^n$, there holds
\begin{equation*}
    \aff\big(\textstyle\bigcup_{i=1}^k (x_i+U_i)\big) = x_1 + \spn\big\{x_2-x_1,\ldots,x_k-x_1\big\} + \textstyle\sum_{i=1}^k \spn(U_i).
\end{equation*}
\begin{proof}
By the definition of the affine hull,    
\begin{equation*}
    \aff\big(\textstyle\bigcup_{i=1}^k (x_i+U_i)\big) = x_1 + \spn\big\{s-x_1:s\in\textstyle\bigcup_{i=1}^k (x_i+U_i)\big\}.
\end{equation*}
If $s\in x_1+U_1$, then $s=x_1+u_1$ for any $u_1\in U_1$ and $s-x_1=u_1$.
For $i\in\{2,\ldots,k\}$, if $s\in x_i+U_i$, then $s=x_i+u_i$ for some $u_i\in U_i$ and 
\begin{equation*}
    s-x_1=(x_i-x_1)+u_i.
\end{equation*}
Therefore,
\begin{align*}
    \{s-x_1:s\in\textstyle\bigcup_{i=1}^k (x_i+U_i)\} &= U_1\cup \textstyle\bigcup_{i=2}^k \{(x_i-x_1)+u_i: u_i\in U_i\},
\end{align*}
which implies our statement, because
\begin{align*}
    \aff\big(\textstyle\bigcup_{i=1}^k (x_i+U_i)\big)
    &= x_1 + \spn\bigl(U_1\cup \textstyle\bigcup_{i=2}^k \{(x_i-x_1)+u_i: u_i\in U_i\}\big) \\
    &= x_1 + \spn\bigl(U_1\cup \big\{x_i-x_1: i\in\{2,\ldots,k\}\big\} + \textstyle\sum_{i=2}^k U_i\bigr) \\
    &= x_1 + \spn\big\{x_i-x_1: i\in\{2,\ldots,k\}\big\} + \textstyle\sum_{i=1}^k U_i.\qedhere
\end{align*}
\end{proof}
\end{lemma}
\begin{lemma}\label{lem:dimension_bound_vpolytopes}
For integers $d,r\in\N$ and virtual polytopes $V_i\in \vPolytopes$, $i\in\{1,\ldots,d\}$, there holds
\begin{equation*}
    \dim\big(\conv\big(\textstyle\sum_{i=1}^d\alpha_{1i}V_i,\ldots,\textstyle\sum_{i=1}^d\alpha_{ri}V_i \big)\big)\leq \big(\textstyle\sum_{i=1}^d\dim V_i\big) + r-1,
\end{equation*}
for any coefficients $\alpha_{1i},\ldots,\alpha_{ri}\in\R$.
\begin{proof}
For $i\in\{1,\ldots,d\}$, let $V_i = P_i-Q_i$ be represented by $P_i\subseteq x_i+U_i$ and $Q_i\subseteq y_i+W_i$ such that $\dim(V_i)=\dim(P_i+Q_i)=\dim(U_i+W_i)$.
For $j\in\{1,\ldots,r\}$, we use the shorthands
\begin{align*}
    P^j&\coloneqq \textstyle\sum_{i=1}^d\alpha_{ji}P_i, & Q^j&\coloneqq \textstyle\sum_{i=1}^d\alpha_{ji}Q_i, \\
    x^j&\coloneqq \textstyle\sum_{i=1}^d\alpha_{ji}x_i, & y^j&\coloneqq \textstyle\sum_{i=1}^d\alpha_{ji}y_i, & z^j\coloneqq x^j+\textstyle\sum_{k=1,k\neq j}^r y^k \\
    U&\coloneqq \textstyle\sum_{i=1}^d U_i, & W&\coloneqq \textstyle\sum_{i=1}^dW_i.
\end{align*}
and observe that
\begin{equation*}
    P^j\subseteq x^j+U\quad\text{and}\quad Q^j\subseteq y^j+W.
\end{equation*}
Therefore,
\begin{align*}
    &\mathrel{\phantom{=}}\dim\big(\conv\big(\textstyle\sum_{ i=1}^d\alpha_{1i}V_i,\ldots,\textstyle\sum_{i=1}^d\alpha_{ri}V_i \big)\big)\\
    &=\dim\big(\conv\big(\textstyle\sum_{i=1}^d\alpha_{1i}(P_i-Q_i),\ldots,\textstyle\sum_{i=1}^d\alpha_{ri}(P_i-Q_i) \big)\big) \\
    &=\dim\big(\conv\big(P^1-Q^1,\ldots,P^r-Q^r \big)\big) \\
    &=\dim\big(\conv\big(P^1+\textstyle\sum_{j=2}^r Q^j\cup\ldots\cup P^r+\sum_{j=1}^{r-1}Q^j \big) - \big(\textstyle\sum_{j=1}^r Q^j\big)\big) \\
    &\leq\dim\big(\conv\big(x^1+\textstyle\sum_{j=2}^r y^j +U+W\cup\ldots\cup x^r+\sum_{j=1}^{r-1}y^j+U+W \big) \hspace{-0.2mm}+\hspace{-0.2mm} \big(\textstyle\sum_{j=1}^r y^j+W\big)\big) \\
    &=\dim\big(\aff\big(z^1+U+W\cup\ldots\cup z^r+U+W \big)\big) \\[-0.92mm]
    &\overset{(\hspace{-0.25mm}*\hspace{-0.25mm})}{=}\dim\big(z^1+\spn\{z^2-z^1,\ldots,z^r-z^1\} + U + W \big) \\
    &\leq\dim(U+W) + r-1 =\dim(\textstyle\sum_{i=1}^d U_i + \textstyle\sum_{i=1}^d W_i) + r-1 \\
    &\leq\big(\textstyle\sum_{i=1}^d\dim(U_i + W_i)\big) + r-1 =\big(\textstyle\sum_{i=1}^d\dim(V_i)\big) + r-1,
\end{align*}
using Lemma~\ref{lem:affhull_subspaces} at $(\hspace{-0.25mm}*\hspace{-0.25mm})$.
\end{proof}
\end{lemma}
\begin{theorem}\label{thm:dim_bound}
For parameters $\ell \in\N_0$, $ \dvec=(d_1,d_2,\ldots,d_\ell)$, and $\rvec=(r_1,\ldots,r_\ell)\in\N^\ell$, any virtual polytope $V \in \dualpolytopesl{}{\ell,\dvec,\rvec}$ satisfies
\begin{equation*}
    \dim(V) \leq \textstyle\sum_{k=1}^{\ell}(r_k-1)\textstyle\prod_{i=k+1}^\ell d_i.
\end{equation*}
\begin{proof}
For $\ell=0$, both sides of the inequality are zero.
We proceed by induction on $\ell \geq 1$, assuming our claim to be true for virtual polytopes in $\dualpolytopesl{}{\ell-1,\dvec,\rvec}$.
Now, consider an arbitrary virtual polytope $V\in\dualpolytopesl{}{\ell,\dvec,\rvec}$.
It can be expressed as
\begin{equation*}
    V=\conv\big(\textstyle\sum_{j=1}^{d_{\ell}}\alpha_{1j}V_j,\ldots,\textstyle\sum_{j=1}^{d_{\ell}}\alpha_{r_\ell j}V_j\big),
\end{equation*}
for suitable virtual polytopes $V_j\in \dualpolytopesl{}{\ell-1,\dvec,\rvec}$ and coefficients $\alpha_{1j},\ldots,\alpha_{r_\ell j}\in\R$.
Invoking Lemma~\ref{lem:dimension_bound_vpolytopes}, there holds 
\begin{align*}
    \dim(V) &= \dim\big(\conv\big(\textstyle\sum_{j=1}^{d_{\ell}}\alpha_{1j}V_j,\ldots,\textstyle\sum_{j=1}^{d_{\ell}}\alpha_{r_\ell j}V_j \big)\big) \\
    &\leq \big(\textstyle\sum_{j=1}^{d_{\ell}}\dim V_j\big) + r_{\ell}-1, \\
    &\leq \big(\textstyle\sum_{j=1}^{d_{\ell}}\textstyle\sum_{k=1}^{\ell-1}(r_k-1)\textstyle\prod_{i=k+1}^{\ell-1} d_i\big) + r_{\ell}-1, \\
    &= \big(d_{\ell}\textstyle\sum_{k=1}^{\ell-1}(r_k-1)\textstyle\prod_{i=k+1}^{\ell-1} d_i\big) + r_{\ell}-1, \\
    &=\textstyle\sum_{k=1}^{\ell}(r_k-1)\textstyle\prod_{i=k+1}^\ell d_i.\qedhere
\end{align*}
\end{proof}
\end{theorem}

The given dimension bound is tight, as is shown next.

\begin{theorem}\label{thm:dimboundattained}
Given parameters $\ell \in\N_0$, $ \dvec=(d_1,d_2,\ldots,d_\ell)$, and $\rvec=(r_1,\ldots,r_\ell)\in\N^\ell$ with $d_i\geq r_i$ for $i\in[\ell]$, if $n\geq\textstyle\sum_{k=1}^{\ell}(r_k-1)\textstyle\prod_{i=k+1}^\ell d_i$, then there exists a polytope $P \in\dualpolytopesl{n}{\ell,\dvec,\rvec}$ such that
\begin{equation*}
    \dim(P) = \textstyle\sum_{k=1}^{\ell}(r_k-1)\textstyle\prod_{i=k+1}^\ell d_i.
\end{equation*}
\end{theorem}
\begin{proof}
Denoting by $\{e_i\}_i^n$ the standard basis of $\R^n$, we prove the following statement by induction on $\ell$.

\textbf{Claim.}
Let $n\geq\textstyle\sum_{k=1}^{\ell}(r_k-1)\textstyle\prod_{i=k+1}^\ell d_i\eqqcolon m_\ell$. Then, for every $v \in \R^n$ and $I\subseteq [n]$ with $|I| = m_\ell$ there is a polytope $P \in \dualpolytopesl{n}{\ell,\dvec,\rvec}$ such that $\aff(P)=v + \spn\{e_i:i \in I\}$.

For $\ell=0$, there is $m_\ell=0$ and thus $I=\emptyset$. In this case, for every $v\in\R^n$, we can just take $P=\{v\}\in\dualpolytopesl{n}{0,\dvec,\rvec}$ whose affine hull is $v+\spn\{e_i:i \in I\}=\{v\}$.

For $\ell\geq 1$, we have that $m_\ell = d_\ell@m_{\ell-1} +(r_\ell-1)$.
Let us take some $v \in \R^n$ and let $I=\{i_1,\ldots,i_{m_\ell}\} \subseteq [n]$ be an arbitrary index set of cardinality $m_\ell$.
We split
\begin{equation*}
    I=I_1\cupdot I_2 \cupdot \ldots \cupdot I_{d_\ell} \cupdot J
\end{equation*}
such that $|I_k|=m_{\ell-1}$ for $k\in[d_\ell]$, $|J|=r_\ell-1$, and denote $J=\{j_1,\ldots,j_{r_\ell-1}\}$.
By induction, for each $k\in[d_\ell]$, there exists a polytope
\begin{equation*}
    P_k\in\dualpolytopesl{n}{\ell-1,\dvec,\rvec}\quad\text{such that}\quad \aff(P_k) = t_k + \spn\{e_i:i\in I_k\},
\end{equation*}
where the reference points $t_k$ are free to be chosen.
Let us pick
\begin{equation*}
    t_{r_\ell}\coloneqq\frac{1}{r_\ell}\bigg(v-\sum\limits_{s=1}^{r_\ell-1}e_{j_s}\bigg)\quad\text{and}\quad t_i\coloneqq t_{r_\ell} + e_{j_i}\quad\text{for }i\in[r_\ell-1]
\end{equation*}
and set $t_k\coloneqq 0$ for $k>r_\ell$.
Note that the points are chosen such that
\begin{equation*}
    \sum\limits_{k=1}^{d_\ell} t_k = \sum\limits_{k=1}^{r_\ell-1} (t_{r_\ell}+e_{j_k}) +t_{r_\ell} = r_\ell t_{r_\ell} + \sum\limits_{k=1}^{r_\ell-1}e_{j_k} = v.
\end{equation*}
Furthermore, for $i\in[r_\ell-1]$, we define the polytopes
\begin{equation*}
    Q_i\coloneqq P_i +\!\!\sum\limits_{\substack{k=1,\;k\neq r_\ell}}^{d_\ell}\!\!\!\!\!\!P_k,\quad Q_{r_\ell}\coloneqq \sum\limits_{k=1}^{d_\ell}P_k,\quad\text{and}\quad P  \coloneqq \conv\left(\bigcup_{i=1}^{r_\ell}Q_i\right).
\end{equation*}
By construction, $P\in\dualpolytopesl{n}{\ell,\dvec,\rvec}$.
It remains to be shown that the affine hull of $P$ is $\aff(P) = v + \spn\{e_i : i\in I\}$.
We compute the affine hulls of the $Q_i$ first.
For $i\in[r_\ell-1]$, we have
\begin{align*}
    \aff(Q_i) &= \aff\big(P_i + \textstyle\sum_{k=1,\,k\neq r_\ell}^{d_\ell}P_k\big) \\
              &= t_i - t_{r_\ell}+ \textstyle\sum_{k=1}^{d_\ell}t_k + \spn\{e_i:i\in I\setminus (J\cup I_{r_\ell})\}\big\} \\
              &= e_{j_i} + v + \spn\{e_i:i\in I\setminus (J\cup I_{r_\ell})\}
\end{align*}
and similarly, $\aff(Q_{r_\ell}) = v + \spn\{e_i:i\in I\setminus J\}$. The affine hull of $P$ is
\begin{align*}
    \aff(P) &= \aff\big(\conv\big(\textstyle\bigcup_{i=1}^{r_\ell}Q_i\big)\big) \\
            &= \aff\big(\textstyle\bigcup_{i=1}^{r_\ell}Q_i\big) \\
            &= \aff\big(\textstyle\bigcup_{i=1}^{r_\ell}\aff(Q_i)\big) \\
            &= \aff\big(\textstyle\bigcup_{i=1}^{r_\ell-1}(e_{j_i} + v + \spn\{e_i:i\in I\setminus (J\cup I_{r_\ell})\}) \\&\hspace{13mm}\cup\aff(v+\spn\{e_i:i\in I\setminus J\})\big) \\[-1mm]
            &\overset{(\hspace{-0.25mm}*\hspace{-0.25mm})}{=} v + \spn\{e_{j_1},\ldots,e_{j_{r_\ell-1}}\} + \spn\{e_i:i\in I\setminus J\}\\
            &= v + \spn\{e_i:i\in I\},
\end{align*}
where Lemma~\ref{lem:affhull_subspaces} is used at $(\hspace{-0.25mm}*\hspace{-0.25mm})$.
This concludes our induction.
\end{proof}

\section{Consequences for the expressivity of sparse networks}\label{sec:expressivity}

Knowing which dimensions can be realized by virtual polytopes from sparse networks, we now explore consequences for the expressivity of corresponding network classes.
The results developed to that point yield the following hierarchy.

\begin{theorem}\label{thm:hierarchydualpolytopesl}
For parameters $\dvec\geq \rvec$, there holds
\begin{equation*}
    \dualpolytopesl{n}{0,\dvec,\rvec}\subsetneq\ldots\subsetneq \dualpolytopesl{n}{\ell',\dvec,\rvec},
\end{equation*}
where $\ell'\in\N$ is the depth for which $n\geq m_{\ell'}\coloneqq \textstyle\sum_{k=1}^{\ell'}(r_k-1)\textstyle\prod_{i=k+1}^{\ell'} d_i$, but $n<m_{\ell'+1}$.
\begin{proof}
The inclusions follow from the recursion in Definition~\ref{def:dualpolytopes}.
The depth $\ell'$ is the largest index for which the dimension bound, Theorem~\ref{thm:dim_bound}, and its attainment, Theorem~\ref{thm:dimboundattained}, apply.
This ensures strictness up to that level.
\end{proof}
\end{theorem}

\begin{corollary}\label{cor:hierarchynetworksl}
Under the conditions of Theorem~\ref{thm:hierarchydualpolytopesl}, there holds
\begin{equation*}
    \networksl{n}{0,\dvec,\rvec}\subsetneq\ldots\subsetneq \networksl{n}{\ell',\dvec,\rvec}.
\end{equation*}
\begin{proof}
This follows directly from Theorem~\ref{thm:hierarchydualpolytopesl} and the duality given by Theorem~\ref{thm:iso_sparse_virtual}.
\end{proof}
\end{corollary}

Our next goal is to show that every indegree constraint is a real constraint and cannot be compensated by width or depth. For this, we aim to show that $\networks{n}{\ell,\dvec,\rvec} \subseteq \MAX_n(m_\ell + 1)$ for a suitable $m_\ell$ that does not depend on the input dimension $n$. We then proceed by showing that for every input dimension $n$ there is a function $f\in \networks{n}{2,\mathbf{2}} \setminus \MAX_n(n)$. Hence, for every $\ell,\mathbf{r},\mathbf{d}$, not all functions representable with a fully-connected network that has only two hidden layers, is in $\networks{n}{\ell,\dvec,\rvec}$.
For the first part, we build on the fact that every polytope admits a triangulation into simplices.
For example, this is discussed by \cite{lee2017subdivisions}. Lemma 9 in \cite{bakaev2025better} implies the following.

\begin{lemma}\label{lem:sum_of_simplices}
Let $P\subseteq \R^n$ be a polytope of dimension $d$.
Then $P$ is a linear combination of simplices of dimension at most $d$.
\end{lemma}

\begin{theorem}\label{thm:inclusion_in_maxterms}
Given parameters $\ell\in\N_0$, $\dvec,\rvec\in\N^\ell$, if $m_\ell\coloneqq\textstyle\sum_{k=1}^{\ell}(r_k-1)\textstyle\prod_{i=k+1}^{\ell} d_i$, then
\begin{equation*}
  \networks{n}{\ell,\dvec,\rvec} \subseteq \MAX_n(m_\ell + 1).
\end{equation*}
\end{theorem}
\begin{proof}
Let be given an arbitrary function $f\in \networks{n}{\ell,\dvec,\rvec}$.
By Theorem~\ref{thm:iso_sparse_virtual}, its Newton polytope $V(f)$ is in $\dualpolytopes{n}{\ell,\dvec,\rvec}$.
The dimension bound, Theorem~\cref{thm:dim_bound}, implies that $V(f)$ is a linear combination of virtual polytopes of dimension at most $m_\ell$.
Together with Lemma \cref{lem:sum_of_simplices}, this says that $V(f)$ is a linear combination of simplices of dimension at most $m_\ell$.
In other words, $V(f)$ decomposes into convex hulls of at most $m_\ell +1$ affinely independent points.
By \cref{lem:iso_cpwl_virtual}, it follows that $f \in M_n(m_\ell + 1)$, proving the claim.
\end{proof}

Our argumentation relies on the following result by \cite{hertrich2021towards}. 
Given a polyhedral complex $\mathcal{P}$, the set of all $k$-dimensional polyhedra in $\mathcal{P}$ is denoted by $\mathcal{P}^{k}$.
\begin{lemma}\label{lem:hyperplanetest}
Let $f\colon\R^n\to\R$ be a convex \cpwl function and let $\mathcal{P}$ be its underlying polyhedral complex.
If there exists a hyperplane $H\subseteq\R^n$ such that the set
\begin{equation*}
    T\coloneqq\!\!\bigcup\limits_{F\in\mathcal{P}^{n-1}:F\subseteq H}\!\!\!\!\!\!\!\!\!\!F
\end{equation*}
is nonempty and contains no line, then $f$ is not contained in $\MAX_n(n)$.
\end{lemma}

\begin{theorem}\label{thm:widthcannotcompensate}
There exists a function $f\in \networks{n}{2,\mathbf{2}} \setminus \MAX_n(n)$.
In particular, for $\ell\geq 2$, $\dvec\geq \mathbf{2}$, $\rvec\in\N^\ell$, and $n\geq m_\ell +1$, $m_\ell\coloneqq\textstyle\sum_{k=1}^{\ell}(r_k-1)\textstyle\prod_{i=k+1}^{\ell} d_i$, there holds
\begin{equation*}
  \networks{n}{\ell,\dvec,\rvec} \subsetneq \networks{n}{\ell,\rvec}.
\end{equation*}
\begin{proof}
By construction, the function
\begin{equation*}
    f\colon\R^n\to \R,\quad f(x) = \max \bigg\{ \sum_{j=1}^{n-2} \max\{0,x_j\},\,\max\{x_{n-1},x_{n}\}\bigg\}
\end{equation*}
is contained in $\networks{n}{2,\mathbf{2}}$, which in turn is a subset of $\networks{n}{\ell,\rvec}$ for $\ell\geq 2$ and $\dvec\geq \mathbf{2}$.

The inclusion $\networks{n}{\ell,\dvec,\rvec} \subseteq \networks{n}{\ell,\rvec}$ holds by definition.
It is strict if $n\geq m_\ell +1$ and $f\notin\MAX_n(n)\supseteq\MAX_n(m_\ell+1)\supseteq \networks{n}{\ell,\dvec,\rvec}$, where the latter inclusion is by Theorem~\ref{thm:inclusion_in_maxterms}.

It remains to be shown that $f\notin\MAX_n(n)$.
We prove that by applying Lemma~\ref{lem:hyperplanetest} to $f$ and the hyperplane $H\coloneqq\{x\in \R^{n}:x_{n-1}=x_{n}\}$.
Certainly, $f$ is convex and \cpwl.
Also, we may rewrite it as
\begin{equation*}
    f(x) = \max \{S(x),\,x_{n-1},\,x_{n}\},\quad\text{where }S(x)\coloneqq \sum_{j=1}^{n-2} \max\{0,x_j\}. 
\end{equation*}
Furthermore, denote by $\mathcal{P}$ the polyhedral complex underlying $f$.
To use Lemma~\ref{lem:hyperplanetest}, we have to verify that $T\coloneqq\textstyle\bigcup_{F\in\mathcal{P}^{n-1}:F\subseteq H} F$ is nonempty and contains no line.

First, we show that there is exactly one $F\in\mathcal{P}^{n-1}$ with $F\subseteq H$, the polyhedron $F^\star\coloneqq \{x\in \R^n:x_{n-1}=x_{n}\geq S(x)\}$.

To see what is in $\mathcal{P}^{n-1}$, let us list all elements of $\mathcal{P}^n$.
The maximal regions on which~$S(x)$ is affine are the sign chambers $\Pi_J\coloneqq\{x\in \R^n: x_j\geq0 \text{ for }j\in J,\,x_j\leq0 \text{ for }j\notin J\}$, $J\subseteq[n-2]$, where $S(x)=\sum_{j\in J}x_j$.
So, $f$ is affine on
\begin{align*}
    C_J&\coloneqq \{x\in\R^n:S(x)\geq x_{n-1},\,S(x)\geq x_{n}\} \cap \Pi_J \quad\text{for }J\subseteq[n-2],\\
    C_{n-1}&\coloneqq\{x\in\R^n:x_{n-1}\geq x_{n},\,x_{n-1}\geq S(x)\}, \\
    C_{n}&\coloneqq\{x\in\R^n:x_{n}\geq x_{n-1},\,x_{n}\geq S(x)\}.
\end{align*}
The polyhedra in $\mathcal{P}^{n-1}$ are intersections of two distinct polyhedra from $\mathcal{P}^n$.
We obtain
\begin{equation*}
C_{n-1}\cap C_{n}=\{x\in\R^n: x_{n-1} = x_{n} \geq S(x)\}=F^\star.
\end{equation*}
So $F^\star\in \mathcal{P}^{n-1}$ and certainly also $F^\star\subseteq H$, thus $F^\star\subseteq T$.

Any other intersection of polyhedra from $\mathcal{P}^n$ imposes further equality constraints.
If different $C_{J_1}$, $C_{J_2}$ intersect, then some $x_j$, $j\in[n-2]$, has to be zero.
If some $C_J$ intersects with $C_{n-1}$, then $S(x)=\textstyle\sum_{j\in J} x_j=x_{n-1}$, and likewise if some $C_J$ intersects $C_{n}$.
With the additional constraint $x_{n-1}=x_{n}$ enforced by $F\subseteq H$, any such intersection has dimension at most $n-2$.
Therefore $F^\star =T=\{x\in\R^n: x_{n-1}=x_{n}\geq S(x)\}$.

It remains to be shown that $T$ contains no line. 
Suppose, for a contradiction, that $L=\{y+\lambda v:\lambda\in \R\}\subseteq T$ for some $y\in\R^n$ and $v\in\R^n$. Setting $t\coloneqq v_{n-1}=v_{n}$, we have
\begin{equation*}
    \sum_{j=1}^{n-2} \max\{0,y_j+\lambda v_j\}= S(y+\lambda v)\leq y_{n-1} + \lambda v_{n-1} = y_{n} + \lambda v_{n} = y_{n} + \lambda t
\end{equation*}
for all~$\lambda\in \R$.
Dividing both sides by $\lambda>0$ yields
\begin{equation*}
    \sum_{j=1}^{n-2} \max\{0,\frac{y_j}{\lambda}+ v_j\} \leq \frac{y_{n}}{\lambda} + t
\end{equation*}
Taking the limit $\lambda\to\infty$, we obtain
\begin{equation*}
    \sum_{j=1}^{n-2} \max\{0,v_j\} \leq t.
\end{equation*}
Analogously, dividing by $\lambda<0$ and taking the limit $\lambda\to-\infty$ yields
\begin{equation*}
    \sum_{j=1}^{n-2} \max\{0,-v_j\} \leq -t.
\end{equation*}
Adding up those inequalities gives
\begin{equation*}
    \sum_{j=1}^{n-2} |v_j| = \sum_{j=1}^{n-2} \big(\max\{0,v_j\} + \max\{0,-v_j\}\big) \leq 0.
\end{equation*}
Thus $v_j=0$ for $j\in[n-2]$.
With this, $S(y+\lambda v)=S(y)$ for all $\lambda\in\R$.
The inequality $S(y)\leq y_n +\lambda t$ must then hold for all $\lambda$, hence letting $\lambda\to\pm\infty$ forces $t=0$ and thus $v=0$.
In other words, there is no line in $T$, concluding the proof.
\end{proof}
\end{theorem}

For the particular setting $\dvec=\rvec=\mathbf{2}$, we obtain the following characterization.
Beyond this strict hierarchy, the class $\networks{n}{\ell,\mathbf{2},\mathbf{2}}$ can be characterized succinctly by relating it to the class $\MAX_n(m)$.
\begin{theorem}\label{thm:Nl22max2^l}
For all $\ell\in\N_0$, there holds
\begin{equation*}
    \networks{n}{\ell, \mathbf{2}, \mathbf{2}}= \MAX_n(2^\ell).
\end{equation*}
\begin{proof}
The inclusion $\networks{n}{\ell,\mathbf{2},\mathbf{2}}\subseteq \MAX_n(2^\ell)$ follows by Theorem~\ref{thm:inclusion_in_maxterms}.
Just note that $m_\ell+1=2^\ell$ for $\dvec=\rvec=\mathbf{2}$.

For the inclusion $\MAX_n(2^\ell)\subseteq \networks{n}{\ell,\mathbf{2},\mathbf{2}}$, consider an arbitrary $f\in \MAX_n(2^\ell)$.
Such a function is of the form
\begin{equation*}
    f(x)=\sum_{i=1}^p \beta_i\max\limits_{j\in[2^\ell]}\{a_{ij}^\top x\}
\end{equation*}
for some $p\in\N, \beta_i\in\R, a_{ij}\in\R^n$.
To check that $f$ is in $\networks{n}{\ell,\mathbf{2},\mathbf{2}}$, let us take a look at Figure~\ref{fig:network}.
Recall that the first layer of networks in $\networks{n}{\ell,\mathbf{2},\mathbf{2}}$ is not indegree constrained.
This allows to realize the linear terms $a_{ij}^\top x$ until the first activation.
The following $\ell$ hidden layers of rank-$2$-maxout units compute $\max_{j\in[2^\ell]}\{a_{ij}^\top x\}$ for $i\in[p]$.
Those terms can be combined linearly, with coefficients $\beta_i$, in the final layer, as desired.
\end{proof}
\end{theorem}
Since $\MAX_n(k) \subsetneq \MAX_n(k+1)$ \citep{hertrich2021towards,koutschan2023representing} for all $k \leq n$, we obtain the following corollary.
\begin{corollary}\label{cor:hierarchynetworks2}
\begin{equation*}
    \networks{n}{0,\mathbf{2},\mathbf{2}}\subsetneq\ldots\subsetneq \networks{n}{\lceil\log_2(n+1)\rceil,\mathbf{2},\mathbf{2}}=\hcpwl_n.
\end{equation*}
\end{corollary}

\section*{Discussion}
All in all, our work shows that the virtual polytopes associated with the functions computable by sparse maxout networks can help understanding them.
We have seen that our bound on their dimension provides a useful tool for analyzing expressivity, yielding various depth hierarchies and showing that width cannot fully compensate for given indegree constraints. The \emph{lineality space} of a positively homogeneous \cpwl function $f \colon \R^n \to \R$ is the vector space $\{v \in \R^n \mid f(x) = f(v+x) \text{ for all } x \in \R^n\}$. The codimension of the lineality space of a function $f$ equals the dimension of the virtual polytope $V_f$. The codimension of the lineality space was also used by \citet{koutschan2023representing} to find the smallest $k$ such that the function $f$ is contained in the vector space $\MAX_n(k)$. Consequently, by tracking the lineality space of a function computed by a sparse network, one could alternatively derive Theorem~\ref{thm:inclusion_in_maxterms}. However, formulating Lemma~\ref{lem:dimension_bound_vpolytopes} directly in terms of the lineality space, rather than in the dual setting of virtual polytopes, seems to be more involved.
\section*{Acknowledgments}
The authors thank Christoph Hertrich for many valuable discussions. 
We gratefully acknowledge partial support by the Deutsche Forschungsgemeinschaft (DFG, German Research Foundation) under Germany's Excellence Strategy -- The Berlin Mathematics Research Center MATH+ (EXC-2046/1, project ID: 390685689). Moritz Grillo was also partially supported by  project 464109215 within the priority programme SPP 2298 “Theoretical Foundations of Deep Learning,” 

\bibliographystyle{plainnat}
\bibliography{bibliography}

\end{document}